\documentclass{article}

    \usepackage[nonatbib, preprint]{nips-2019}

\usepackage[utf8]{inputenc} % allow utf-8 input
\usepackage[T1]{fontenc}    % use 8-bit T1 fonts
\usepackage{hyperref}       % hyperlinks
\usepackage{url}            % simple URL typesetting
\usepackage{booktabs}       % professional-quality tables
\usepackage{amsfonts}       % blackboard math symbols
\usepackage{nicefrac}       % compact symbols for 1/2, etc.
\usepackage{microtype}      % microtypography

\usepackage{amsmath,amsfonts,bm}
\usepackage{bbold}

\usepackage{xstring}
\usepackage{dsfont}
\def\alphabet{abcdefghijklmnopqrstuvwxyzABCDEFGHIJKLMNOPQRST123456789}
\renewcommand{\vec}[1]{
\IfSubStr{\alphabet}{#1}{
\ensuremath{\mbf{\MakeLowercase{#1}}}
}{
\ensuremath{\bm{\MakeLowercase{#1}}}
}
}

\newcommand{\mat}[1]{
\IfSubStr{\alphabet}{#1}{
\ensuremath{\mbf{\MakeUppercase{#1}}}
}{
\ensuremath{\bm{\MakeUppercase{#1}}}
}
}
\newcommand{\set}[1]{\ensuremath{\mathcal{\MakeUppercase{#1}}}}

\def\eqref#1{equation~\ref{#1}}
\def\veta{{\vec{\eta}}}

\def\va{{\vec{a}}}

\def\vw{{\vec{w}}}
\def\vx{{\vec{x}}}

\def\vz{{\vec{z}}}

\def\mA{{\mat{A}}}

\def\mW{{\mat{W}}}
\def\mX{{\mat{X}}}

\def\sA{{\set{A}}}

\def\sC{{\set{C}}}

\def\sF{{\set{F}}}
\def\sG{{\set{G}}}

\def\sN{{\set{N}}}

\def\sS{{\set{S}}}

\def\sW{{\set{W}}}
\def\sX{{\set{X}}}
\def\sY{{\set{Y}}}

\usepackage{amssymb, amsthm}
\newtheorem{theorem}{Theorem}
\newtheorem{lemma}{Lemma}

\newtheorem{corollary}{Corollary}

\newtheorem{definition}{Definition}
\def\T{\mathsf{T}}
\def\R{\mathbb R}

\def\E{\mathbb E}
\def\P{\mathbb P}
\def\N{\mathbb N}

\def\ord{{\mcl O}}

\def\Bern{\mathrm{Bern}}
\def\Var{\mathrm{Var}}
\def\Cov{\mathrm{Cov}}
\def\ball{\set{B}}
\def\1{\vec{1}}

\newcommand{\mbf}{\mathbf}

\newcommand{\mcl}{\mathcal}
\newcommand{\mrm}{\mathrm}

\newcommand*{\inner}[2]{\left\langle#1,#2\right\rangle}
\newcommand*{\norm}[1]{\left\|#1\right\|}
\newcommand*{\parents}[1]{\left(#1\right)}
\newcommand*{\braces}[1]{\left\{#1\right\}}
\newcommand*{\brackets}[1]{\left[#1\right]}
\newcommand*{\card}[1]{\left|#1\right|}
\newcommand*{\abs}[1]{\left|#1\right|}

\renewcommand\epsilon{\varepsilon}
\renewcommand\hat{\widehat}
\renewcommand\tilde{\widetilde}
\renewcommand\bar{\overline}

\DeclareMathOperator*{\argmax}{argmax}

\def\margin{\ell}
\def\datadist{\mathcal D}

\def\deq{:=}
\newcommand*{\ind}[1]{\mathbb{1}\left(#1\right)}

\usepackage{algorithm,algorithmic}

 \usepackage{graphicx}
 \usepackage{wrapfig}
\graphicspath{{./figures/}}
\title{Adversarial Risk Bounds for Neural Networks through Sparsity based Compression}

\author{%
  Emilio Rafael ~Balda, Arash ~Behboodi, Niklas Koep, Rudolf ~Mathar
  \\
  Institute for Theoretical Information Technology\\
  RWTH Aachen University\\
  \texttt{\{balda, behboodi, koep, mathar\}@ti.rwth-aachen.de} \\
}

\begin{document}

\maketitle

\begin{abstract}
Neural networks have been shown to be vulnerable against minor adversarial perturbations of their inputs, especially for high dimensional data under $\ell_\infty$ attacks.
To combat this problem, techniques like adversarial training have been employed to obtain models which are robust on the training set.
However, the robustness of such models against adversarial perturbations may not generalize to unseen data.
To study how robustness generalizes, recent works assume that the inputs have bounded $\ell_2$-norm in order to bound the adversarial risk for $\ell_\infty$ attacks with no explicit dimension dependence.
In this work we focus on $\ell_\infty$ attacks on $\ell_\infty$ bounded inputs and prove margin-based bounds.
Specifically, we use a compression based approach that relies on efficiently compressing the set of tunable parameters without distorting the adversarial risk. 
To achieve this, we apply the concept of effective sparsity and effective joint sparsity on the weight matrices of neural networks.
This leads to bounds with no explicit dependence on the input dimension, neither on the number of classes.
Our results show that neural networks with approximately sparse weight matrices not only enjoy enhanced robustness, but also better generalization.
\end{abstract}

\section{Introduction}
In recent years, neural networks have been shown to be particularly vulnerable to maliciously designed perturbations of their inputs. 
Such perturbed inputs are known as adversarial examples and they are often only slightly distorted versions of the original inputs. 
For example, in image classification, adversarial examples have been shown to be indistinguishable from the original image to the human eye. 
This phenomena motivated several works aimed at understanding the nature of classifiers, and in particular neural networks, in the presence of adversarial examples. 

The first work discussing the vulnerability of neural networks to adversarial examples was presented by Goodfellow et al. \cite{harness_nnn}. 
In that work, the authors hypothesized that this phenomena can be explained by the excessive linearity of trained neural networks.
However, such claim was refuted by various subsequent works.
For instance, in \cite{TanayG16} it is shown that it is possible to train linear classifiers that are resistant to adversarial attacks which stands in contrast to the linearity hypothesis. 
Moreover, it is exemplified that high dimensional problems are not necessarily more sensitive to adversarial examples. 
Further, \cite{sabour2015adversarial} manipulated deep representations instead of the input and argued that the linearity hypothesis is not sufficient to explain this type of attack. 
In \cite{NIPS2016_6331}, the authors suggest that the flatness of the decision boundary is a reason for the existence of adversarial examples.  
In \cite{flexi}, the authors propose the low flexibility of neural networks, compared to the difficulty of the classification task, as a reason for the existence of adversarial examples. 
Another perspective is proposed in \cite{TanayG16} with the {boundary tilting} mechanism. It is argued that adversarial examples exist when the
decision boundary lies close to the sub-manifold of sampled data. 
The notion of {adversarial strength} is introduced which refers to the deviation angle between the target classifier and the nearest centroid classifier. 
It is shown that the adversarial strength can be arbitrarily increased independently of the classifier's accuracy by tilting the boundary.  
\cite{RozsaGB16b} give another explanation arguing that over the course of training the correctly classified samples do not have a significant impact on shaping the decision boundary and eventually remain close to it. This phenomenon is called {evolutionary stalling}.
In \cite{Rozsa2016AreAA}, the correlation between robustness and accuracy of the classifier is studied empirically by attacking different state-of-the-art neural networks. It is observed that higher accuracy neural networks are more sensitive to adversarial attacks than lower accuracy ones. 
Other works like \cite{moosavi-dezfooli2018robustness} study the curvature of the decision boundaries around training samples. 
They argue that neural networks are particularly vulnerable to universal perturbations in shared directions along which the decision boundary is systematically positively curved.
In \cite{scale, madry2018towards}, evidence was given that increasing model capacity alone can help to make neural networks more robust against adversarial attacks. Further, it was observed that robust models (obtained by adversarial training) exhibit rather sparse weights compared to non-robust ones.
While these approaches contributed to understanding the nature of adversarial examples, they do not consider whether the robustness of classifiers against adversarial perturbations generalizes to unseen data.

If a classifier is robust to perturbations of the training set, can we guarantee that it will be also robust to perturbations of the test set?
This question is not particularly new. 
The optimization community has studied this problem for quite some time. 
The work of Xu, et al. \cite{Xu2008RobustRA}, studied robust regression in Lasso, while  
later work \cite{xu2009robustness} obtained results for support vector machines. 
Other works considered the generalization properties of robust optimization in a distributional sense \cite{Sinha2018CertifyingSD}, that is when adversarial examples are assumed to be samples from the worst possible distribution within a Wasserstein ball around the original one. 
As discussed, these works provide algorithms for training various types of classifiers with robustness guarantees. 
Regarding neural networks, for the case where no adversarial perturbations are present, there exists an extensive literature on their generalization guarantees.
Many of these works are based on bounding the Rademacher complexity of the function class
\cite{Bartlett2017SpectrallynormalizedMB,Golowich2018SizeIndependentSC,Neyshabur2018TowardsUT,Li2018OnTG}, 
while others make use of the PAC-Bayes framework 
\cite{neyshabur2017pac,Neyshabur2017ExploringGI,2018DeterministicPG}. 
There other works which rely in different techniques, for instance 
 \cite{Arora2018StrongerGB} relies on compressing the weights of neural networks.
Despite this knowledge, proving robustness guarantees for neural networks remained unstudied till recently. 
Initial works going into this direction studied neural networks in artificial scenarios. 
For instance, Attias et al. \cite{Attias2018ImprovedGB}  proved generalization bounds for the case when the adversary can modify a finite number of entries per input. 
Following this approach, Diochnos et al. \cite{Diochnos2018AdversarialRA} showed that the number of flipped bits required to fool almost all inputs is less than $\ord (\sqrt n)$, for the case when the input is binary and uniformly distributed. 
As similar subsequent result \cite{Mahloujifar2018CanAR} for binary inputs, proved the existence of {polynomial-time attacks} that find adversarial examples of Hamming distance $\mcl O (\sqrt n)$. 
Concurrently, the work of Schmidt et al. \cite{Schmidt2018AdversariallyRG} showed that the amount of data necessary to classify $n$-dimensional Gaussian data grows by a factor of $\sqrt n$ in the presence of an adversary. 
However, Cullina et al. \cite{Cullina2018PAClearningIT} showed that the VC-dimension of linear classifiers does not increase in the adversarial setting. Additionally, they derived generalization guarantees for binary linear classifiers.
Moreover, Montasser et. al. \cite{Montasser2019VCCA} showed that VC-classes are learnable in the adversarial setting, but only if we refrain from using standard empirical risk minimization approaches. 
Later works considered more general scenarios. 
Using a PAC-Bayes approach, Farnia et al. \cite{Farnia2019GeneralizableAT} proved a generalization bound for neural networks under $\ell_2$ attacks.
However, deriving bounds for attacks with bounded $\ell_\infty$-norm (instead of $\ell_2$-norm) is of particular interest, since most successful attacks in computer vision are of this type. 
In addition, such attacks tend to be more effective for scenarios where the input dimension is large, thus deriving generalization bounds without explicit dimension dependence is promising. 
Some recent have advanced into . Since those works are closely related to this paper, we discussed them more in detail in the following section. 
\subsection{Related Work} % (fold)
\label{subsec:related_work}
The following works address the problem of proving generalization bounds for neural networks in the adversarial setting, where the attacker has bounded $\ell_\infty$ perturbations.
\begin{itemize}
	\item Yin et al. \cite{Yin2018RademacherCF} bounded the Rademacher complexity for linear classifiers and neural networks in the adversarial setting. This lead to explicit bounds on the notion of \textit{adversarial risk} for the linear classifier, as well as neural networks. 
	Nevertheless, such bound applied only to neural networks with one hidden layer and relu activations.
	\item Concurrent work from Khim et al. \cite{Khim2018AdversarialRB} proved bounds on a surrogate of the  adversarial risk. In that work, the authors use the so-called \textit{tree transform} on the function class to derive their results. Under the assumption that the original inputs have $\ell_2$ bounded norm, the authors prove generalization bounds with no explicit dimension dependence in the binary classification setting. However, the authors extend this to $K$-class classification by incurring an additional factor $K$ on their bound.
	\item Later work from Tu et al. \cite{Tu2018TheoreticalAO} formulated generalization in the adversarial setting as a   minimax problem. Their proposed framework is more general than previous ones in the sense that it can be applied to support vector machines and principal component analysis, as well as neural networks. However, for neural networks this approach yielded a generalization bound with explicit dimension dependence. 
\end{itemize}
One common assumption shared by these works is that the inputs come from a distribution with bounded $\ell_2$-norm, which is a weaker notion than assuming $\ell_\infty$ bounded inputs. 
% 
% 
% 

% \begin{table}[h]
% 	\begin{tabular}{c|c|c}
% 		Authors & Data Assumption & Sample Complexity \\
% 		\hline
% 		\cite{Khim2018AdversarialRB} & $\norm{\vx}_2 \leq R$ & $\tilde\ord\parents{d K^2 \prod_{j=2}^d \norm{\mW^j}_{1,\infty}^2 (R \max_{j} \norm{\mW^j}_\F + \epsilon \norm{\mW^1}_{1,\infty})^2 }$ \\
% 		% \hline
% 		\cite{Tu2018TheoreticalAO} & $\norm{\vx}_2 \leq R$ & $\tilde\ord\parents{\frac 1 \gamma R n \prod_{j=1}^d \norm{\mW^j}_2^2 \parents{\sum_{j=1}^d \sqrt{\frac{\norm{\mW^j}_\F}{\norm{\mW^j}_2}} }^2}$  \\
% 		% \hline
% 		Here & $\norm{\vx}_\infty \leq 1$ &
% 		$\tilde\ord\parents{d \parents{\frac{ 1 + \gamma -  \epsilon }{ \gamma - 2 \epsilon }}^2 \parents{ \sum_{j=1}^d \sqrt{ 
% 		\frac{\norm{\mW^j}_{1/2,\infty} \norm{\mW^j}_{1,1}}{\norm{\mW^j}_{1,\infty}^2} 
% 		} }^2}$
% 		\\
% 	\end{tabular}
% \end{table}

% subsection  (end)

% section related_work (end)

\subsection{Our Contributions}

In this work, we study the problem of bounding the generalization error of multi-layer neural networks under $\ell_\infty$ attacks, where we assume that the original inputs have $\ell_\infty$ bounded norm. 
Using a compression approach, we obtain bounds with no explicit dependence on the input dimension or the number of classes. 
We summarize our contributions as follows.
\begin{itemize}
	\item We prove generalization bounds in the presence of adversarial perturbations of bounded $\ell_\infty$-norm under the assumption that the input distribution has bounded $\ell_\infty$-norm as well. This is an improvement with respect to recent works where the input is assumed to be $\ell_2$ bounded. 
	\item We extended the compression approach from \cite{Arora2018StrongerGB} by incorporating the notion of effective sparsity. Using this technique we
	prove that the capacity of neural networks, under adversarial perturbations, is bounded by the effective sparsity and effective joint sparsity of its weight matrices. 
	This result has no explicit dimension dependence, neither it depends on the number of classes. We show that approximately sparse weights not only improve robustness against $\ell_\infty$ bounded adversarial perturbations, but they provide better generalization as well. 
	\item 
	% As in \cite{Arora2018StrongerGB}, we show experiments where this bound on the capacity is  less than the VC-dimension of the neural network. Moreover, 
	We corroborate our result with a small experiment on the MNIST dataset, where the bound correlates with adversarial risk. 
	We observe that adversarial training significantly decreases the bound, while standard training does not. 
	Similarly, adversarial training seems to decrease both, effective sparsity and effective joint sparsity, as predicted by our result. 
\end{itemize}

\subsection{Notation} 
We introduce first the notation used in this chapter and some of the basic definitions needed throughout this chapter. The letters $\vec{x},\vec{y},\dots$ are used for vectors, $\mat{A},\mat{B},\dots$ for matrices and  $\set{X},\set{Y},\dots$ for sets. 
We denote the set $\{1,\dots,n\}$ by $[n]$ for $n\in\N$. 
For any vector $\vec{x}=(x_1,\dots,x_n)^{\T}\in\R^n$ and $p > 0$, the $\ell_p$-norm of $\vec{x}$ is denoted by $\norm{\vx}_p$. The notation $\ball_{p,\epsilon}^n$ is used to refer to an $n$-dimensional $\ell_p$ ball of size $\epsilon$, that is the set $\ball_{p, \epsilon}^{n} = \{\vx \in \R^n: \|\vx\|_p \leq \epsilon\}$. 
% 
% For any matrix $\mA \in \R^{n_1 \times n_2}$ and $p>0$, its Shatten $p$-norm is denoted by $\norm{\mA}_p$ and its equal to the $\ell_p$-norm of the singular value vector of singular values of $\mA$. 
% 
For any matrix $\mA = (\va_1, \dots, \va_{n_2}) \in \R^{n_1 \times n_2}$ and $p,q>0$, its operator $p$-norm is denoted by $\norm{\mA}_p$ and its mixed $(p,q)$-norm by $\norm{\mA}_{p,q}$. These norms are given by 
\[
\norm{\mA}_p = \sup_{\norm{\vx}_p \leq 1} \norm{\mA \vx}_p \, 
\quad \text{and} \quad \,
\norm{\mA}_{p,q} = \norm{ \parents{ \norm{\va_1}_p , \dots, \norm{\va_{n_2}}_p }^\T }_q \, . 
\]
Finally, we use the compact notation $\tilde \ord (n) \deq \ord (n \log n) $ to ignore logarithmic factors.

\section{Problem Setup} % (fold)
\label{sec:definitions}

We start with the standard margin-based statistical learning framework. 
Let $\sX$ be the feature space, $\sY$ the label space, and $\datadist: \sX \times \sY \rightarrow [0,1]$ a probability measure. 
In this work, it is assumed that all instances $\vx \in \sX$ have $\ell_\infty$-norm bounded by $1$, that is $\sX \subseteq \ball_{\infty, 1}^{n} \subset \R^n$. 
Without loss of generality, let the label space be $\sY = \{1, 2, \dots, \card{\sY}\}$.
Using these notions, a classifier is defined through its so called {score function} $f: \R^n \rightarrow \R^{\card{\sY}}$ such that the predicted label is $\argmax_{j \in \sY} [f(\cdot)]_j$. 
Moreover, given an instance $(\vx, y) \in \sX \times \sY$, the {classification margin} is defined as 
\[
\margin (f; \vx, y) = [f (\vx)]_y - \max_{j\neq y} [f (\vx)]_j \, .
\]
In this manner, a positive margin implies correct classification.
Then, for any distribution $\datadist$ the expected margin loss with margin $\gamma \geq 0$ is defined as
$
L_{\gamma}(f) = \P_{(\vx, y) \sim \datadist} \left[ \margin (f; \vx, y)  \leq \gamma \right] \, .
$
% $\widehat L_{\gamma}(f)$ denotes an empirical estimate of $L_{\gamma}(f)$.
% Note that $L_{0}(f)$ and $\widehat L_{0}(f)$ are the expected risk and training error.

In this paper, we study the case where an adversary is present. This adversary has access to the input $\vx$ and is allowed to add a perturbation $\veta$ with $\ell_\infty$-norm bounded by some $\epsilon \geq 0$ (i.e., $\veta \in \ball_{p,\epsilon}^n$) such that the classification margin is as small as possible. 
This perturbed input $\vx + \veta$ is usually known as an {adversarial example}. 
Furthermore, let us define the margin under adversarial perturbations as 
\[
\margin_\epsilon (f; \vx, y) = \inf_{\veta \in \ball^n_{\infty, \epsilon}} \margin(f; \vx + \veta, y) \, .
\] 
This leads to the definition of {adversarial margin loss}: 
\[
L_{\gamma}^{\epsilon}(f) = \P_{(\vx, y) \sim \datadist} \left[ \margin_\epsilon (f; \vx, y) \leq \gamma \right] \, .
\]
Let $\sS =\{(\vx_1, y_1), \dots, (\vx_m, y_m)\}$ be the training set composed of $m$ instances drawn independently from $\datadist$. 
Using these instances we define $\widehat L_{\gamma}^{\epsilon}(f) = \frac 1m \sum_{i=1}^m \ind{ \margin_\epsilon (f; \vx_i, y_i) \leq \gamma}$ as the empirical estimate of $L_{\gamma}^{\epsilon}(f)$, where $\ind{\cdot}$ denotes the indicator function. 
Note that $L_{0}^{\epsilon}(f)$ and $\widehat L_{0}^{\epsilon}(f)$ are the expected risk and training error under adversarial perturbations, respectively.

For many classifiers, such as deep neural networks, the score function $f$ belongs to a complicated function class $\sF$, which usually has more sample complexity than the size of the training set.  
Even without the presence of an adversary, it is challenging to bound the {generalization error}, given by the difference $L_{0}(f) - \hat L_{\gamma}(f)$, of such function classes. 
The key idea behind the compression framework presented in \cite{Arora2018StrongerGB} is to show that there exists a finite function class $\sG$ with low sample complexity and a mapping that assigns a function $g\in \sG$ to every $f\in \sF$ such that the empirical loss is not severely degraded. 
This trick allows us to bound the generalization error using the sample complexity of $\sG$ instead of $\sF$. 
A drawback of this method is that we are only able to bound $L_{0}(g) - \hat L_{\gamma}(f)$ instead of the true generalization error original. 
Nevertheless, as the authors mentioned in \cite{Arora2018StrongerGB}, a similar issue is present as well in standard PAC-Bayes bounds, where the bound is on a noisy version of $f$. 
Moreover, the authors discuss some possible ways to solve this issue, but these approaches were left for future work.  
In this paper we leverage such a compression framework by extending it to the case when an adversary is present. 
Our goal is to bound the generalization error under the presence of an adversary.  
We start by introducing some formal definitions and theorems, similar to the ones in \cite{Arora2018StrongerGB}. 
All proofs are deferred to the supplementary material.

\begin{definition}[$(\gamma, \epsilon, \sS)$-compressible]
Given a set of parameter configurations $\sA$, let $\sG_\sA = \{g_A | A \in \sA\}$ be a set of parametrized functions $g_A$. We say that the score function $f \in \sF$ is $(\gamma, \epsilon, \sS)$-compressible through $\sG_\sA$ if
\[
\forall \vx \in \sS, y \in \sY : |\margin_\epsilon(f; \vx, y) - \margin_\epsilon(g_A; \vx, y)| \leq \gamma .
\]  
\end{definition}

\begin{theorem}\label{thm:bound-main}
Given the finite sets $\sA$ and $\sG_\sA = \{g_A | A \in \sA\}$, if $f$ is $(\gamma, \epsilon, \sS)$-compressible via $\sG_\sA$  then there exists $A \in \sA$ such that with high probability
\[
L_0^\epsilon(g_A) \leq \widehat L_\gamma^\epsilon(f) + \ord\left(\sqrt{\frac{\log \card{\sA}}{m}} \right) \, .
\]
\end{theorem}

\begin{corollary}\label{cor:bound-main}
  In the same setting of Theorem \ref{thm:bound-main}, if $f$ is compressible only for a fraction $1 - \delta$ of the training sample, then with high probability
  \[
  L_0^\epsilon(g_A) \leq \widehat L_\gamma^\epsilon(f) +  \ord\left(\sqrt{\frac{\log \card{\sA}}{m}} \right) + \delta\, .
  \]
\end{corollary}
This main definition and following theorems are trivial extensions of the ones used in \cite{Arora2018StrongerGB} to the adversarial setting. 
However, even for the linear classifier, the main technique used in that work for compressing $f$ cannot be applied to the setup of this paper without incurring into explicit dimensionality dependencies in the resulting bounds. 
This will be explained in detail in the next section.

% 
% section definitions (end)
% 
\section{Main Results}
In this section we introduce our main results. We start with linear classifiers on binary classification and move forward to neural networks and multi-class classification. 
\subsection{Linear Classifier} % (fold)
\label{subsec:linear}
We start with a linear classifier for binary labels. 
Assume that $\vx \in \ball_{\infty,1}^n$, $y \in \{1,2\}$ and let $\vw  = (w_1,\dots, w_n)^\T$ be a vector of weights of a linear classifier. 
Then the score function of the linear classifier is given by
\[
	f_\vw(\vx) = \begin{pmatrix} 0 \\ \inner{\vw}{\vx} \end{pmatrix} \, .
\]
This simplifies the margin to $\margin(f; \vx, y) = (2y -3) \inner{\vw}{\vx}$, which leads to 
\[
	\margin_\epsilon(f_\vw; \vx, y) = (2y -3) (\inner{\vw}{\vx} - \epsilon\norm{\vw}_1) \, .
\]
Note that $(2y -3) \in \{-1,+1\}$. 
The weight vector $\vw \in \R^n$ of this classifier, with margin $\gamma$, can be compressed into another $\hat \vw$ such that both classifiers make the same predictions with reasonable probability (see supplementary material). 
Given $\delta \in (0,1]$, the compressed classifier $\hat \vw$ is constructed entry-wise as $\hat w_i = z_i w_i /p_i$ where $p_i = (1+\epsilon)^2 \abs{w_i} / \delta \gamma^2$ and $z_i \sim \Bern\parents{p_i}$. 
Such classifier $\hat \vw$ outputs the same prediction as $\vw$ with probability $1- \delta$ and has only $\ord \parents{(\log n) (1+\epsilon^2)/\delta \gamma^2}$ non-zero entries with high probability. 
By discretizing $\hat \vw$ we obtain a compression setup that maps $\vw$ into a discrete set but fails with probability $\delta$. 
Therefore, we can apply Corollary \ref{cor:bound-main} and choose $\delta = \parents{(1+\epsilon)^2 / \gamma^2 m}^{1/3}$, which yields a generalization bound of order $\tilde \ord\parents{((1+\epsilon)^2 / \gamma^2 m)^{1/3}}$ (see the supplementary material for more details). 
This approach is fairly similar to the original one in the work of Arora et al. \cite{Arora2018StrongerGB}, but the $p_i$ values are chosen differently in order to deal with the new term $\epsilon\norm{\vw}_1$ that appears in the margin's expression. 
The result of the previous section provides a dimension-free bound\footnote{Except for logarithmic terms.}. 
However, that bound scales with $m^{1/3}$ instead of $\sqrt{m}$ since the compression approach fails with probability $\delta$. 
To tackle this issue, Arora et al. \cite{Arora2018StrongerGB} proposed a compression algorithm based on random projections. 
In their setup, this technique works due to a famous corollary of Johnson-Lindenstrauss lemma that shows that we can construct random projections which preserve the inner $\inner{\vw}{\vx}$. 
In addition, since the Euclidean inner product can be induced by the $\ell_2$-norm, the $\ell_2$-norm of $\vw$ is preserved as well. 
However, in this setup we would need a random projection that preserves $\norm{\vw}_1$ and $\inner{\vw}{\vx}$ at the same time, which seems unattainable unless additional assumptions are made.
We therefore propose to assume an effective sparsity bound on $\vw$, which is defined as follows.
\begin{definition}[Effective $\bar s$-sparsity]
  A vector $\vw \in \R^n$ is effectively $\bar s$-sparse, 
  with $\bar s \in [1,n]$, if
  \[
  \norm{\vw}_{1/2} \leq \bar s  \norm{\vw}_1 \, .
  \]
\end{definition}
Note that all $s$-sparse vectors are effectively $s$-sparse as well, but not vice-versa. Assuming that $\vw$ is effectively sparse allows us to compress it by simply setting its lowest entries to zero. The following lemma provides a tight bound on the error, in the $\ell_1$ sense, that is caused by this process. 
\begin{lemma}[\cite{rauhut}: Theorem 2.5]\label{lmm:sparse-approx-error}
  % For any $p>q>0$ and $\vec x \in \R^n$,
  % \[
  % \inf\braces{\norm{\vec x - \vec z}_p : \vec z \text{ is $s$-sparse}} \leq \frac{1}{s^{1/q - 1/p}} \norm{\vec x}_q \, .
  % \]
  For any $\vec w \in \R^n$ the following inequalities hold:
  \begin{align*}
  	\inf\braces{\norm{\vec w - \vec z}_1 : \vec z \text{ is $s$-sparse}} &\leq \frac{1}{4s} \norm{\vec w}_{1/2} \, , \\
  	\inf\braces{\norm{\vec w - \vec z}_\infty : \vec z \text{ is $s$-sparse}} &\leq \frac{1}{s} \norm{\vec w}_{1} \, .
  \end{align*}
  
In both cases the infimum is attained when $\vz$ is an $s$-sparse vector whose non-zero entries are the $s$-largest absolute entries of $\vec w$. 
\end{lemma}
For any effectively $\bar s$-sparse classifier $\vw$ with margin $\gamma$, this lemma allows us to compress it into a vector $\hat \vw$, with only $\ord(\bar s (1+\epsilon)/\gamma)$ non-zero entries, such that the both classifiers assign the same label to any input. 
Since this compression approach does not fail, we can discretize $\hat \vw$ and apply Theorem \ref{thm:bound-main}.
This allows to prove the following generalization bound for the linear classifier in the presence of an adversary. 
\begin{theorem}\label{thm:linear-sparse-bound}
Let $\vw$ be any linear classifier with $\norm{\vw}_{1/2} / \norm{\vw}_{1} \leq \bar s$, and margin $\gamma >0$ on the training set $\sS$. Then, if $\card{\sS} = m$, with high probability the adversarial risk is bounded by
  \[
  L_0^\epsilon(f_{\hat \vw}) \leq \widehat L_\gamma^\epsilon(f_{\vw}) + 
  \tilde\ord\left( \sqrt{\frac{(1+\epsilon) \bar s}{\gamma m}} \right)  \, ,
  \]
  where $\tilde \ord(\cdot)$ ignores logarithmic factors.
\end{theorem}
This result provides a bound with no explicit dimension dependence. 
Moreover, we observe that the presence of an adversary only increases the sample complexity by a factor $(1+\epsilon)$.  
% 
% subsection effectively_sparse_linear_classifier (end)

\subsection{Neural Networks} % (fold)
\label{subsec:neural_networks}

Due to the $\ell_\infty$-norm bound on the perturbation $\veta$, in this work the mixed $(1,\infty)$-norm of the weight matrices plays a central role. 
As an example let us consider a linear classifier in multi-class classification, that is $f(\vx) = \mW^\T \vx$. Then, a perturbation $\veta$ can perturb the score function at most
\[
\sup_{\norm{\veta}_\infty \leq 1} \norm{\mW^\T \veta}_\infty 
= \norm{\mW^\T}_\infty = \norm{\mW}_{1,\infty} \, . 
\]
The last equality comes from the properties of operator norms, see \cite{tropp2004topics} for more details.
Similar statements can be made for the layers of a neural network with $1$-Lipschitz activation functions. 
Let us start by defining a $d$-layered fully connected neural network as  
\begin{equation}\label{eq:nn-def}
\vx^i \deq \phi(\mW^{i^\T}\vx^{i-1}) \, , \quad \forall i=1,2,\dots,d \, ,
\end{equation}
where $\phi$ is a $1$-Lipschitz activation function applied entry-wise, $\vx^0 \deq \vx $ and $ f(\vx) \deq \vx^d$. 
Then, the following lemma allows us to quantify how much error is incurred by perturbing the input of a layer, or by switching the matrix $\mW$ to a different one. 
\begin{lemma}\label{lmm:lipschitz}
  If $\phi$ is a $1$-Lipschitz activation function, then for any $\mW, \hat \mW$ the following inequalities hold
  \begin{align*}
    \norm{\phi(\mW^\T \vx) - \phi(\mW^\T \parents{\vx + \veta})}_\infty &\leq \norm{\mW}_{1,\infty} \norm{\veta}_\infty\, , \\
    \norm{\phi(\mW^\T \vx) - \phi(\hat\mW^\T \vx)}_\infty & \leq \norm{\mW - \hat{\mW}}_{1,\infty} \norm{\vx}_\infty \, .
  \end{align*}
\end{lemma}

Following the steps of Section \ref{subsec:linear}, we now impose some conditions on $\mW$ that allow us to efficiently compress it into another matrix $\hat \mW$ which belongs to a potentially small set. To that end, let us start by introducing the notion of effective joint sparsity.
\begin{definition}[Effective joint sparsity]
  A matrix $\mW \in \R^{n_1 \times n_2}$ is effectively joint $\bar s$-sparse, 
  with $\bar s \in [1,n_2]$, if
  \[
  	\norm{\mW}_{1,1} \leq \bar s  \norm{\mW}_{1,\infty} \, .
  \]
\end{definition}
Any matrix with $s$ non-zero columns is effectively joint $\bar s$-sparse as well. 
Given a matrix $\mW = (\vw_1, \dots , \vw_n)$ is effectively joint $\bar s$-sparse if and only if effective joint sparsity can be seen as a type of effective sparsity condition on the vector $(\norm{\vw_1}_1 , \dots, \norm{\vw_{n}}_1 )^\T$. 
A consequence of Lemma \ref{lmm:sparse-approx-error} is that we can compress effectively joint-sparse matrices by setting to zero their columns with lowest $\ell_1$-norm. For example, assume that $\mW \in \R^{n_1 \times n_2}$ is an effective joint $\bar s$-sparse matrix and that $\hat \mW$ is constructed by setting to zero all columns of $\mW$ except for its $s$ largest in the $\ell_1$ sense. 
Then, by Lemma \ref{lmm:sparse-approx-error}, we can bound the $\norm{\cdot}_{1,\infty}$ error as 
\[
\norm{\mW - \hat \mW}_{1,\infty} \leq \frac{1}{s}\norm{\mW}_{1,1} 
\leq \frac{\bar s}{s}\norm{\mW}_{1,\infty} \, .
\]
The resulting compressed matrix $\hat \mW$ would have only $s$ non-zero columns instead of the original $n_2$. 
However, every column has potentially $n_1$ non-zero values. 
In order to compress $\mW$ further we assume that each one of its columns has bounded effective sparsity as well. 
In summary, effective joint sparsity allows us to reduce the number of non-zero columns in a matrix, while effective sparsity of the columns allows us to reduce the number of non-zero elements that each of the non-zero columns may have.  
Finally, discretization is handled using a standard covering number argument. 
Putting all together into the following compression algorithm (Algorithm \ref{alg:compress}) allows us to map $\mW$ into a discrete set while keeping the $\norm{\cdot}_{1,\infty}$ error bounded.
\begin{algorithm}[htb]
  \begin{algorithmic}
    \STATE \textbf{Require:} $\gamma > 0$ and $\mW \in \R^{ n_1 \times n_2}$ with $\norm{\mW}_{1,\infty}=1$, effectively $\bar s_1$-sparse columns and is effectively joint $\bar s_2$-sparse
    \STATE \textbf{Ensure:} 
    \[
    \norm{\mW - \hat \mW}_{1,\infty} \leq \gamma \, ,
    \]
    where $\hat \mW$ belongs to a discrete set $\sW$ such that $\log \card{\sW} \leq \tilde \ord \parents{\norm{\mW}_{1,\infty}^2 \bar s_1 \bar s_2 / \gamma^2 }$
    \STATE Choose $s_1 = 3\norm{\mW}_{1,\infty} \bar s_1 / 4\gamma $ and $s_2 = 3 \norm{\mW}_{1,\infty} \bar s_2 / \gamma $
    \STATE Let $\bar \mW \in \R^{ n_1 \times n_2}$ be obtained by setting to zero the columns of $\mW$ except for the $s_2$ columns with largest $\ell_1$ norm
    \STATE Let $\tilde \mW \in \R^{ n_1 \times n_2}$ be constructed by keeping the $s_1$ largest values of every column of $\bar \mW$ and setting to zero the other entries
    \STATE Let $\sW$ be the set all possible $\tilde \mW$
    % \STATE Let $\hat \mW$ be the matrix with where each column is the closest point in the $\ell_1$ sense belonging to the covering set $\sC(\gamma/3, \ball_{1,1}^{n_1} \cap \ball_{0,s_1}^{n_1}, \norm{\cdot}_1)$
    \STATE Let $\hat \mW$ be the closest matrix in the $\norm{\cdot}_{1,\infty}$ sense to the covering set of $\sW$ with $\norm{\tilde \mW - \hat \mW}_{1,\infty} \leq \gamma/3$
    \STATE \textbf{Return:} $\hat \mW$ 
  \end{algorithmic}
  \caption{$\mathrm{MatrixCompress}\parents{\cdot, \gamma}$}
  \label{alg:compress}
\end{algorithm}

By construction, using this algorithm guarantees that the error is bounded, as stated in the following lemma.  
\begin{lemma}\label{lmm:nn-layer-capacity}
Let $\mW$ be an effectively joint $\bar s_2$-sparse matrix with effectively $\bar s_1$-sparse columns, such that $\norm{\mW}_{1,\infty} \leq 1$. If $\hat \mW = \mathrm{MatrixCompress}\parents{\mW, \gamma}$, then
\[
  \norm{\mW - \hat \mW}_{1,\infty} \leq \gamma \, ,
\]
where $\hat \mW$ belongs to a discrete set $\sC$ such that $\log \card{\sC} \leq \tilde \ord \parents{\norm{\mW}_{1,\infty}^2 \bar s_1 \bar s_2 / \gamma^2 }$.
\end{lemma}
From this lemma we can see that the set of possible compressed matrices has reasonable size. 
Moreover, approximately sparse matrices can be compressed efficiently. 
This result leads us to the main contribution of this paper, which is stated in the following theorem.
\begin{theorem}\label{thm:nn-bound}
Assume $\vx \in \ball_{\infty, 1}^n$. Let $f_{\mW}$ be a $d$-layer neural network with ReLU activations, and effectively joint $\bar s_2^j$-sparse weight matrices with effectively $\bar s_1^j$-sparse columns for $j=1,\dots, d$. 
Let us assume that the network is rebalanced so that $\norm{\mW^1}_{1,\infty} = \cdots = \norm{\mW^d}_{1,\infty} = 1$. 
Then, given $\gamma > 0$ and $\epsilon < \gamma/4$, there exists a finite function set $\sG$ composed of the functions $f_{\hat \mW}$ such that for any $f_\mW$ the adversarial risk is bounded as
  \[
    L_0^\epsilon(f_{\hat \mW}) \leq \widehat L_\gamma^\epsilon(f_\mW) + \tilde \ord\left(\sqrt{
    \frac{d}{m} \parents{\frac{ 1 + \gamma/2 -  \epsilon }{ \gamma/2 - 2 \epsilon }}^2 \parents{ \sum_{j=1}^d \sqrt{\bar s_1^j \bar s_2^j} }^2
    } \right) \, 
  \]
  with high probability.
\end{theorem}
This result proves a bound with no explicit dimension dependence, which is also independent from the number of classes. 
On the other hand, there seems to be an unavoidable dependence with $\sqrt d$.
However, this dependence is also present in the bounds for multi-layer neural networks, derived in related works \cite{Khim2018AdversarialRB,Tu2018TheoreticalAO}.

Finally, we conduct a experiment to corroborate these findings. 
To that end, we train a fully connected neural network of $3$ layers with ReLU activations on the MNIST dataset. 
After preprocessing, the inputs are $1024$-dimensional vectors with $\ell_\infty$ norm bounded by one. 
\begin{wrapfigure}{r}{0.5\linewidth}
	\centering
	\includegraphics[width=0.99\linewidth]{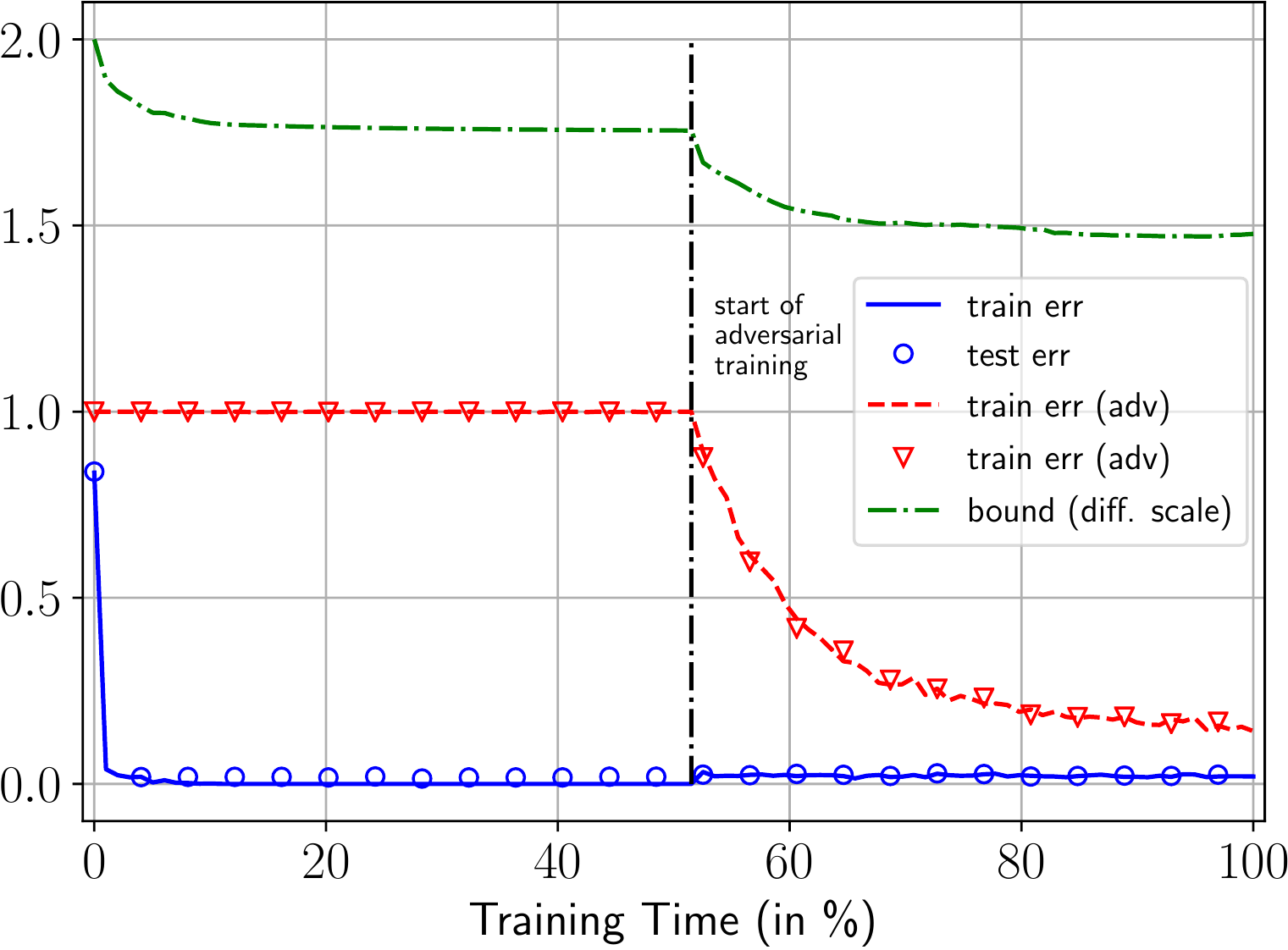}
	\caption{Generalization bound for a vanilla neural network on the MNIST dataset. Adversarial training improves the bound while standard training does not.}
	\label{fig:mnist-bound}
\end{wrapfigure}
The weight matrices are of size $1024 \times 500$, $500 \times 150$ and $150 \times 10$.
To estimate the adversarial risk, we use the projected gradient descent (PGD) attack \cite{pgd_attack} with $\ell_\infty$-norm bounded by $0.2$ and perturbations computed through $10$ iterations of the PGD algorithm. 
This PGD method is the state of the art algorithm for adversarial training.
In figure \ref{fig:mnist-bound} the network is first trained without using adversarial examples. Then, after $50\%$ of the training time, we start introducing adversarial examples to training set. 
These is carried out using the PGD method as described above, except for $0.2$ bound on the perturbation's $\ell_\infty$-norm. 
Instead, we start with a $0.05$ norm bound and slowly increase it until reaching $0.2$.
The script for this experiment is given as supplementary material. 
We can see our result from Theorem $\ref{thm:nn-bound}$ correlates well with the adversarial risk, as it starts decreasing when adversarial training begins.
\begin{figure*}[bh]
	\centering
	\includegraphics[width=0.49\linewidth]{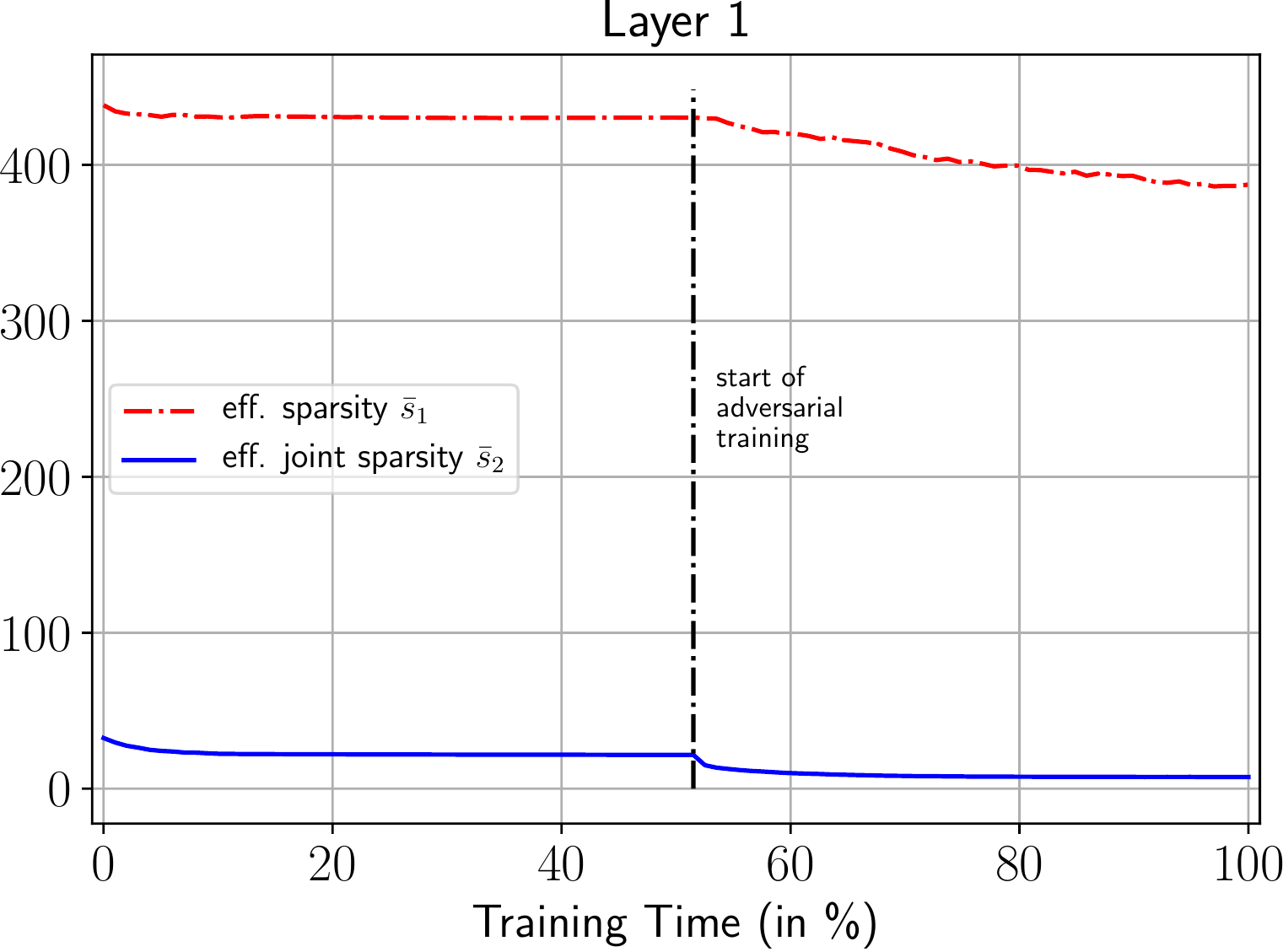}
	\includegraphics[width=0.49\linewidth]{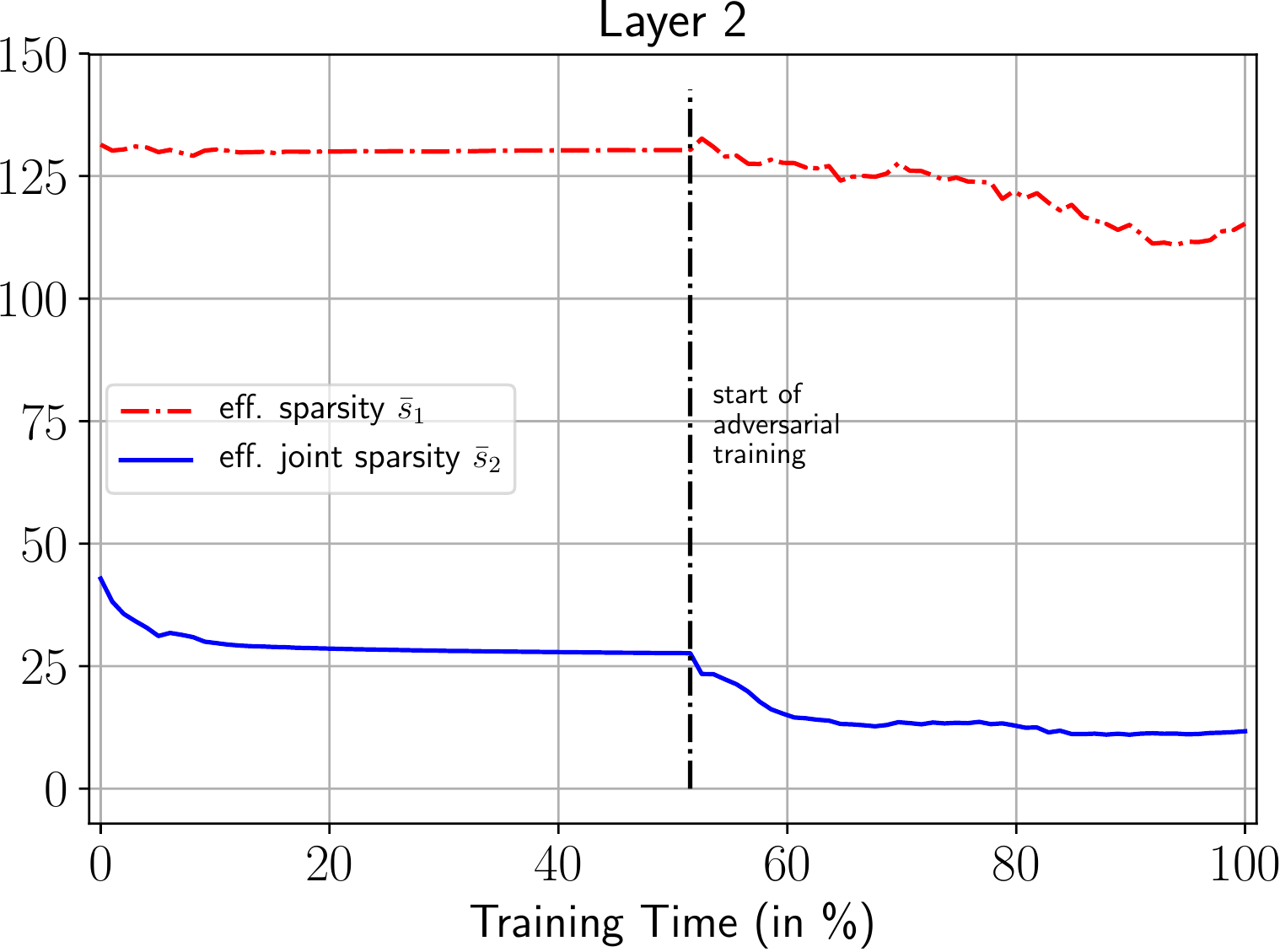}

	\includegraphics[width=0.49\linewidth]{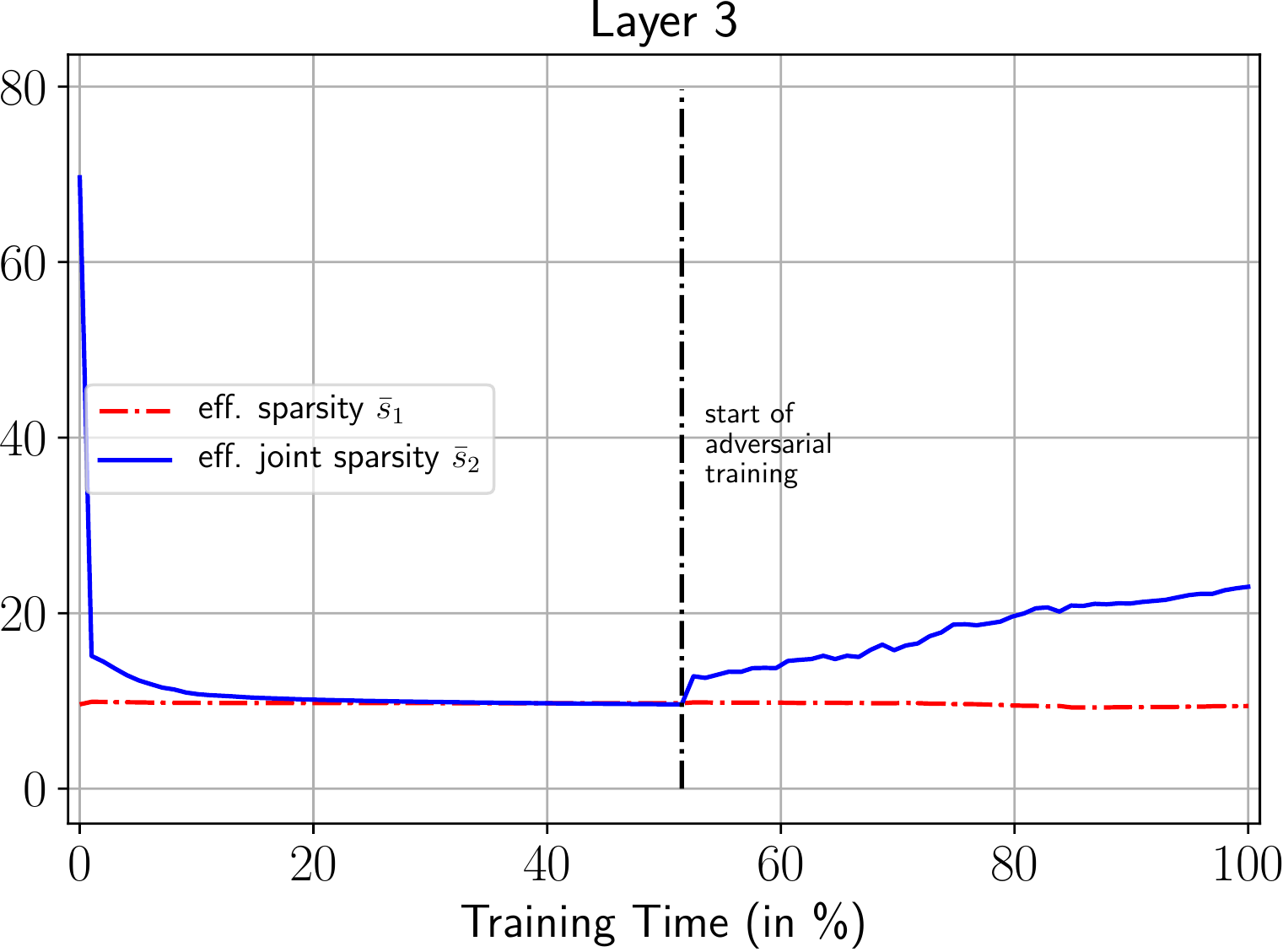}
	\caption{Experiment on the MNIST dataset. Effective sparsity and effective joint sparsity of the weight matrices, at every layer, of a vanilla neural network. These quantities tend to improve with adversarial training.}
	\label{fig:nn-norms}
\end{figure*}
Additionally, we compute the effective sparsity and effective joint sparsity of the weight matrices. In Figure \ref{fig:nn-norms}, we see how these quantities correlate well with the adversarial risk as well. 
These findings show that inducing sparsity on weight matrices  does not only provide robustness, it also improves generalization of neural networks.

%%%%%%%%%%%%%%%%%%%%%%%%%%%%%%%%%%%%%%%%%
% \input{sparse-low-rank}
%%%%%%%%%%%%%%%%%%%%%%%%%%%%%%%%%%%%%%%%%

% section a_compression_approach (end)
% 

\cleardoublepage
\small
\bibliographystyle{plain}
\bibliography{main}

\begin{thebibliography}{10}

\bibitem{Arora2018StrongerGB}
Sanjeev Arora, Rong Ge, Behnam Neyshabur, and Yi~Zhang.
\newblock Stronger generalization bounds for deep nets via a compression
  approach.
\newblock In {\em ICML}, 2018.

\bibitem{Attias2018ImprovedGB}
Idan Attias, Aryeh Kontorovich, and Yishay Mansour.
\newblock Improved generalization bounds for robust learning.
\newblock In {\em ALT}, 2018.

\bibitem{Bartlett2017SpectrallynormalizedMB}
Peter~L. Bartlett, Dylan~J. Foster, and Matus Telgarsky.
\newblock Spectrally-normalized margin bounds for neural networks.
\newblock In {\em NIPS}, 2017.

\bibitem{Cullina2018PAClearningIT}
Daniel Cullina, Arjun~Nitin Bhagoji, and Prateek Mittal.
\newblock Pac-learning in the presence of evasion adversaries.
\newblock {\em CoRR}, abs/1806.01471, 2018.

\bibitem{Diochnos2018AdversarialRA}
Dimitrios~I. Diochnos, Saeed Mahloujifar, and Mohammad Mahmoody.
\newblock Adversarial risk and robustness: General definitions and implications
  for the uniform distribution.
\newblock {\em CoRR}, abs/1810.12272, 2018.

\bibitem{Farnia2019GeneralizableAT}
Farzan Farnia, Jesse~M. Zhang, and David Tse.
\newblock Generalizable adversarial training via spectral normalization.
\newblock {\em CoRR}, abs/1811.07457, 2019.

\bibitem{flexi}
Alhussein Fawzi, Omar Fawzi, and Pascal Frossard.
\newblock Analysis of classifiers' robustness to adversarial perturbations.
\newblock 02 2015.

\bibitem{NIPS2016_6331}
Alhussein Fawzi, Seyed-Mohsen Moosavi-Dezfooli, and Pascal Frossard.
\newblock Robustness of classifiers: from adversarial to random noise.
\newblock In D.~D. Lee, M.~Sugiyama, U.~V. Luxburg, I.~Guyon, and R.~Garnett,
  editors, {\em Advances in Neural Information Processing Systems (NIPS)},
  pages 1632--1640. Curran Associates, Inc., 2016.

\bibitem{rauhut}
Simon Foucart and Holger Rauhut.
\newblock {\em A Mathematical Introduction to Compressive Sensing}.
\newblock Birkh\"{a}user, 2013.

\bibitem{Golowich2018SizeIndependentSC}
Noah Golowich, Alexander Rakhlin, and Ohad Shamir.
\newblock Size-independent sample complexity of neural networks.
\newblock In {\em COLT}, 2018.

\bibitem{harness_nnn}
Ian Goodfellow, Jonathon Shlens, and Christian Szegedy.
\newblock Explaining and harnessing adversarial examples.
\newblock In {\em International Conference on Learning Representations (ICLR)},
  2015.

\bibitem{Khim2018AdversarialRB}
Justin Khim and Po-Ling Loh.
\newblock Adversarial risk bounds via function transformation.
\newblock 2018.

\bibitem{scale}
Alexey Kurakin, Ian~J. Goodfellow, and Samy Bengio.
\newblock Adversarial machine learning at scale.
\newblock In {\em International Conference on Learning Representations (ICLR)},
  2017.

\bibitem{Li2018OnTG}
Xingguo Li, Junwei Lu, Zhaoran Wang, Jarvis~D. Haupt, and Tuo Zhao.
\newblock On tighter generalization bound for deep neural networks: Cnns,
  resnets, and beyond.
\newblock {\em CoRR}, abs/1806.05159, 2018.

\bibitem{madry2018towards}
Aleksander Madry, Aleksandar Makelov, Ludwig Schmidt, Dimitris Tsipras, and
  Adrian Vladu.
\newblock Towards deep learning models resistant to adversarial attacks.
\newblock In {\em International Conference on Learning Representations (ICLR)},
  2018.

\bibitem{pgd_attack}
Aleksander Madry, Aleksandar Makelov, Ludwig Schmidt, Dimitris Tsipras, and
  Adrian Vladu.
\newblock Towards {Deep} {Learning} {Models} {Resistant} to {Adversarial}
  {Attacks}.
\newblock In {\em International Conference on Learning Representations}, 2018.

\bibitem{Mahloujifar2018CanAR}
Saeed Mahloujifar and Mohammad Mahmoody.
\newblock Can adversarially robust learning leverage computational hardness?
\newblock {\em CoRR}, abs/1810.01407, 2018.

\bibitem{Montasser2019VCCA}
Omar Montasser, Steve Hanneke, and Nathan Srebro.
\newblock Vc classes are adversarially robustly learnable, but only improperly.
\newblock {\em CoRR}, abs/1902.04217, 2019.

\bibitem{moosavi-dezfooli2018robustness}
Seyed-Mohsen Moosavi-Dezfooli, Alhussein Fawzi, Omar Fawzi, Pascal Frossard,
  and Stefano Soatto.
\newblock Robustness of classifiers to universal perturbations: A geometric
  perspective.
\newblock In {\em International Conference on Learning Representations (ICLR)},
  2018.

\bibitem{2018DeterministicPG}
Vaishnavh Nagarajan and Zico Kolter.
\newblock Deterministic pac-bayesian generalization bounds for deep networks
  via generalizing noise-resilience.
\newblock 2018.

\bibitem{Neyshabur2017ExploringGI}
Behnam Neyshabur, Srinadh Bhojanapalli, David McAllester, and Nathan Srebro.
\newblock Exploring generalization in deep learning.
\newblock In {\em NIPS}, 2017.

\bibitem{neyshabur2017pac}
Behnam Neyshabur, Srinadh Bhojanapalli, David McAllester, and Nathan Srebro.
\newblock A pac-bayesian approach to spectrally-normalized margin bounds for
  neural networks.
\newblock {\em arXiv preprint arXiv:1707.09564}, 2017.

\bibitem{Neyshabur2018TowardsUT}
Behnam Neyshabur, Zhiyuan Li, Srinadh Bhojanapalli, Yann LeCun, and Nathan
  Srebro.
\newblock Towards understanding the role of over-parametrization in
  generalization of neural networks.
\newblock {\em CoRR}, abs/1805.12076, 2018.

\bibitem{Rozsa2016AreAA}
Andras Rozsa, Manuel G{\"u}nther, and Terrance~E. Boult.
\newblock Are accuracy and robustness correlated.
\newblock {\em 2016 15th IEEE International Conference on Machine Learning and
  Applications (ICMLA)}, pages 227--232, 2016.

\bibitem{RozsaGB16b}
Andras Rozsa, Manuel G{\"{u}}nther, and Terrance~E. Boult.
\newblock Towards robust deep neural networks with {BANG}.
\newblock {\em CoRR}, abs/1612.00138, 2016.

\bibitem{sabour2015adversarial}
Sara Sabour, Yanshuai Cao, Fartash Faghri, and David~J Fleet.
\newblock Adversarial manipulation of deep representations.
\newblock In {\em International Conference on Learning Representations (ICLR)},
  2016.

\bibitem{Schmidt2018AdversariallyRG}
Ludwig Schmidt, Shibani Santurkar, Dimitris Tsipras, Kunal Talwar, and
  Aleksander Madry.
\newblock Adversarially robust generalization requires more data.
\newblock In {\em NeurIPS}, 2018.

\bibitem{Sinha2018CertifyingSD}
Aman Sinha, Hongseok Namkoong, and John~C. Duchi.
\newblock Certifying some distributional robustness with principled adversarial
  training.
\newblock In {\em ICLR}, 2018.

\bibitem{TanayG16}
Thomas Tanay and Lewis~D. Griffin.
\newblock A boundary tilting persepective on the phenomenon of adversarial
  examples.
\newblock {\em CoRR}, abs/1608.07690, 2016.

\bibitem{tropp2004topics}
Joel~Aaron Tropp.
\newblock {\em Topics in sparse approximation}.
\newblock PhD thesis, 2004.

\bibitem{Tu2018TheoreticalAO}
Zhuozhuo Tu, Jingwei Zhang, and Dacheng Tao.
\newblock Theoretical analysis of adversarial learning: A minimax approach.
\newblock {\em CoRR}, abs/1811.05232, 2018.

\bibitem{Xu2008RobustRA}
Huan Xu, Constantine Caramanis, and Shie Mannor.
\newblock Robust regression and lasso.
\newblock {\em IEEE Transactions on Information Theory}, 56:3561--3574, 2008.

\bibitem{xu2009robustness}
Huan Xu, Constantine Caramanis, and Shie Mannor.
\newblock Robustness and regularization of support vector machines.
\newblock {\em Journal of Machine Learning Research}, 10(Jul):1485--1510, 2009.

\bibitem{Yin2018RademacherCF}
Dong Yin, Kannan Ramchandran, and Peter Bartlett.
\newblock Rademacher complexity for adversarially robust generalization.
\newblock {\em CoRR}, abs/1810.11914, 2018.

\end{thebibliography}

%%%% MOVE TO SUPP-MATERIAL BEFORE SUBMISSION %%%%%%%%%%%%%%%%%%
\cleardoublepage
\appendix
\section{Deferred Proofs}
\begin{proof}(Theorem \ref{thm:bound-main}) 
  Since $\widehat L_0^\epsilon(g_A)$ is an average of $m$ i.i.d random variables with expectation equal to $L_0^\epsilon(g_A)$ we may use Hoeffdingen's inequality, yielding
  \[
  \P_{(\vx, y) \sim \datadist} \left[ \widehat L_0^\epsilon(g_A) - L_0^\epsilon(g_A) \geq \tau \right] \leq \exp \left( -2m \tau^2 \right) \,.
  \]
  Note that $|\sA| = \exp(\log \card{\sA})$. 
  Then, let us choose $\tau = \sqrt{\frac{\log \card{\sA}}{m}}$ and take an union bound over all $A \in \sA$, leading to  
  \[
  \P_{(\vx, y) \sim \datadist} \left[ \widehat L_0^\epsilon(g_A) - L_0^\epsilon(g_A) \geq \sqrt{\frac{\log \card{\sA}}{m}} \right] 
  \leq \exp(\log \card{\sA}) \exp \left( -2\log \card{\sA} \right) =  
  \exp (-\log \card{\sA}) \, .
  \]
  Since $f$ is $(\gamma, \epsilon, \sS)$-compressible via $g$, then  
  \[
  \forall \vx \in \sS: \quad |\margin_\epsilon(f;\vx,y) - \margin_\epsilon(g_A;\vx,y)| \leq \gamma \, ,
  \]
  which implies that 
  \[
  \widehat L_0^\epsilon(g_A) \leq \widehat L_\gamma^\epsilon(f) \, .
  \]
  Combining these results we get that 
  \[
  L_0^\epsilon(g_A) \leq \widehat L_\gamma^\epsilon(f) + \ord\left(\sqrt{\frac{\log \card{\sA}}{m}} \right) 
  \]
  with probability at least $1- \exp (-\log \card{\sA}) = 1 - 1/\card{\sA}$, which we consider as high probability.
\end{proof}
\begin{definition}[$\mathrm{CompressVector}(\gamma, \vw)$]
  Given $\vw \in \ball_{1,1}^n, \delta \in (0,1], \gamma >0$ and $ \epsilon>0$, let us define the random mapping $\mathrm{CompressVector}(\gamma, \cdot)$ which outputs $\hat \vw = (\hat w_1, \dots, \hat w_n)^\T = \mathrm{CompressVector}(\gamma, \vw)$ as follows
  \[
    \hat w_i = z_i w_i/p_i\, , 
    \quad \text{with} \quad 
    z_i \sim \Bern(p_i)
    \, \text{ and } 
    p_i = \frac{|w_i|}{\delta \gamma^2}(1+\epsilon)^2
    \, ,
  \]
  where $\Bern(p_i)$ denotes the Bernoulli distribution with probability $p_i$.
\end{definition}
% 
% 
% 

% subsection linear_classifier (end)

\begin{lemma}\label{lmm:vec-compress}
Given $\vw \in \ball_{1,1}^n, \delta \in (0,1], \gamma >0$ and $ \epsilon>0$.
If $\hat \vw = \mathrm{CompressVector}(\gamma, \vw)$ then
  \[
  \forall \vx \in \ball_{\infty,1}^n, \, y\in \sY :\quad 
  \P_{\hat \vw}\left[ 
  \abs{\margin_\epsilon(f_\vw; \vx, y) - \margin_\epsilon(f_{\hat \vw}; \vx, y) } \geq \gamma 
  \right] \leq \delta \, ,
  \]
  and the number of non-zero entries in $\hat \vw$ is less than $\ord((\log n) (1+\epsilon)^2 / \delta \gamma^2)$ with high probability.
\end{lemma}
\begin{proof}(of Lemma \ref{lmm:vec-compress})
  
  Note that $\E [\hat w_i] = \frac{w_i}{p_i} \E[z_i] = w_i$ thus $\E [\hat \vw] = \vw$. 
  Similarly, $\E [|\hat w_i|] = \abs{\frac{w_i}{p_i}} \E[z_i] = |w_i|$ and since $\hat w_i$'s are independent we get $\E [\norm{\hat \vw}_1] = \norm{\vw}_1$. 
  This implies that 
  \[
  \E \margin_\epsilon(f_{\hat \vw}; \vx, y) = \E \brackets{ \inner{\hat\vw}{\vx} - \epsilon\norm{\hat \vw}_1 }
  = \inner{\vw}{\vx} - \epsilon\norm{\vw}_1 
  = \margin_\epsilon(f_{\vw}; \vx, y)  \, .
  \]
  Now lets compute the variance of $\hat w_i$ as 
  \[
  \Var\brackets{\hat w_i} = \E\brackets{\hat w_i^2} - \E\brackets{\hat w_i}^2 = (w_i/p_i)^2 p_i - w_i^2 = \frac{1-p_i}{p_i}w_i^2 \, .
  \]
  The same calculation yields 
  \[
  \Var\brackets{|\hat w_i|} = \frac{1-p_i}{p_i}w_i^2 \, .
  \]
  The covariance between $\abs{\hat w_i}$ and $ \hat w_i$ is
  \[
  \Cov\parents{\abs{\hat w_i}, \hat w_i} = \E\brackets{|\hat w_i|\hat w_i} - \E[|\hat w_i|]\E[\hat w_i]
  = \frac{1-p_i}{p_i}|w_i|w_i \, .
  \]
  Now putting all together we get
  \begin{align*}
    \Var\brackets{\hat w_i x_i - \epsilon |\hat w_i|} 
    &= x_i^2 \Var[\hat w_i] - 2 \epsilon x_i \Cov(\hat w_i, |\hat w_i|) +\epsilon^2 \Var[|w_i|^2] \\
    &= \frac{1-p_i}{p_i}\parents{x_i^2 w_i^2 - 2 \epsilon x_i |w_i|w_i + \epsilon^2 w_i^2} \\
    &\leq \frac{w_i^2}{p_i}\parents{x_i^2 + 2 \epsilon |x_i| + \epsilon^2} \\
    &= \frac{\delta \gamma^2}{(1+\epsilon)^2}|w_i| \parents{x_i^2 + 2 \epsilon |x_i| + \epsilon^2} \, .
  \end{align*}
  Since $\hat w_i$'s are independent, we get
  \begin{align*}
    \Var\brackets{\inner{\hat\vw}{\vx} - \epsilon\norm{\hat \vw}_1} 
    &= \Var\brackets{\sum_{i=1}^n \hat w_i x_i - \epsilon |\hat w_i|} \\
    &= \sum_{i=1}^n \Var\brackets{\hat w_i x_i - \epsilon |\hat w_i|} \\
    &\leq \frac{\delta \gamma^2}{(1+\epsilon)^2} \sum_{i=1}^n |w_i| \parents{x_i^2 + 2 \epsilon |x_i| + \epsilon^2} \\
    &= \frac{\delta \gamma^2}{(1+\epsilon)^2}  \parents{\inner{|\vw|}{\vx^2} + 2 \epsilon \inner{|u|}{|c|}  +\epsilon^2 \norm{\vw}_1} &\text{($\vx^2$ is entry-wise)}\\
    &\leq \frac{\delta \gamma^2}{(1+\epsilon)^2} \parents{\norm{\vw}_1 \norm{\vx^2}_\infty 
    + 2 \epsilon \norm{\vw}_1 \norm{\vx}_\infty 
    + \epsilon^2 \norm{\vw}_1} \\
    &\leq \frac{\delta \gamma^2}{(1+\epsilon)^2} (1 + 2 \epsilon + \epsilon^2) = \delta \gamma^2 \,.
  \end{align*}
  By Chebyshev's inequality we get
  \[
  \P\brackets{| (\inner{\hat\vw}{\vx} - \epsilon\norm{\hat \vw}_1) - \inner{\vw}{\vx} - \epsilon\norm{\vw}_1 | > \gamma} \leq \delta \, .
  \]
  On the other hand the expected number of non-zero entries in $\hat \vw$ is given by 
  \[
  \E\brackets{\norm{\hat \vw}_0} = \sum_{i=1}^n p_i = \sum_{i=1}^n \frac{|w_i|}{\delta \gamma^2}(1+\epsilon^2) = \frac{(1+\epsilon)^2}{\delta \gamma^2} \, .
  \] 
  Then, by Hoefdingen's inequality the number of non-zero entries in $\hat \vw$ is less than $\ord((\log n) (1+\epsilon)^2 / \delta \gamma^2)$ with high probability.
\end{proof}
Now we handle discretization by clipping and then rounding in the following lemma.
\begin{lemma}\label{lmm:vec-compress-round}
  Let us define
  \begin{itemize}
    \item $\vw' $ component-wise as $w_i' = w_i \ind{|w_i| \geq \frac{\gamma}{4n(1+\epsilon)}}$,
    \item $\tilde \vw = \mathrm{CompressVector}(\gamma/2, \vw')$,
    \item $\hat \vw$ is obtained by rounding each entry of $\tilde \vw$ to the nearest multiple of $\frac{\gamma}{2n(1+ \epsilon)}$. 
  \end{itemize}
  Then we have that
  \[
  \forall \vx \in \ball^n_{\infty,1}, \, y\in \sY :\quad \P_{\hat \vw}\left[ \abs{\margin_\epsilon(f_\vw; \vx, y) - \margin_\epsilon(f_{\hat \vw}; \vx, y) } \geq \gamma \right] \leq \delta \, .
  \] 
\end{lemma}
\begin{proof} (of Lemma \ref{lmm:vec-compress-round})
  We start by bounding the error incurred by clipping, that is
  \begin{align*}
    \abs{\margin_\epsilon(f_{\vw}; \vx, y) - \margin_\epsilon(f_{\vw'}; \vx, y) } 
    \leq& \abs{\inner{\vw}{\vx} - \inner{\vw'}{\vx}} + \epsilon \abs{\norm{\vw}_1 - \norm{\vw'}_1}
    \\
    \leq& \abs{\inner{\vw - \vw'}{\vx}} + \epsilon\norm{\vw - \vw'}_1
    \\
    \leq& \norm{\vw - \vw'}_1 \norm{\vx}_\infty + \epsilon\norm{\vw - \vw'}_1
    \\
    \leq &\norm{\vw - \vw'}_1 (1+ \epsilon) 
    \\
    \leq& \frac{\gamma}{4n(1+\epsilon)}n(1+\epsilon) = \gamma/4 \, . 
  \end{align*}
  Similarly, the error incurred by discretizing $\tilde \vw$ is bounded by
  \begin{align*}
    \abs{\margin_\epsilon(f_{\tilde \vw}; \vx, y) - \margin_\epsilon(f_{\hat \vw}; \vx, y) } 
    \leq &\norm{\tilde \vw - \hat\vw}_1 (1+ \epsilon) 
    \\
    \leq& \frac{\gamma}{2n(1+\epsilon)}\frac n2 (1+\epsilon) = \gamma/4 \, . 
  \end{align*}
  By Lemma \ref{lmm:vec-compress} we know that with probability at least $1 - \delta$ we have that $\abs{\margin_\epsilon(f_{\tilde \vw}; \vx, y) - \margin_\epsilon(f_{\hat \vw}; \vx, y) } \leq \gamma/2$.
  Combining these three results yields
  \begin{align*}
    \abs{\margin_\epsilon(f_\vw; \vx, y) - \margin_\epsilon(f_{\hat \vw}; \vx, y) } 
    \leq \quad
    &\abs{\margin_\epsilon(f_{\vw}; \vx, y) - \margin_\epsilon(f_{\vw'}; \vx, y) }
    \\ 
    + &\abs{\margin_\epsilon(f_{\vw'}; \vx, y) - \margin_\epsilon(f_{\tilde \vw}; \vx, y) }
    \\
    + &\abs{\margin_\epsilon(f_{\tilde \vw}; \vx, y) - \margin_\epsilon(f_{\hat \vw}; \vx, y) }
    \\
    \leq \quad  &\gamma/4 + \gamma/2 + \gamma/4 \leq \gamma
  \end{align*}
  with probability at least $1 - \delta$.
\end{proof}

\begin{theorem}\label{thm:linear-bound}
  With high probability
  \[
    L_0^\epsilon(f_{\hat \vw}) \leq \widehat L_\gamma^\epsilon(f_\vw) + 
    \tilde \ord \parents{\parents{\frac{(1+\epsilon)^2}{\gamma^2 m}}^{1/3}} \, , 
  \]
  where $\tilde \ord(\cdot)$ ignores logarithmic factors.
\end{theorem}
\begin{proof} (of Theorem \ref{thm:linear-bound}]) 
  Let $\sA$ be the set of vectors with at most $\ord((\log n) (1+\epsilon)^2 / \delta \gamma^2)$ non-zero entries, where each entry is a multiple of $2 \gamma / 2n(1+ \epsilon)$ between $-\delta \gamma^2 / (1+\epsilon)^2$ and $\delta \gamma^2 / (1+\epsilon)^2$.
  Then $\card{\sA} = r^q$ with
  \[
  r = 2 \frac{\delta \gamma^2 / (1+\epsilon)^2}{2 \gamma / 2n(1+ \epsilon)}
  = \frac{4n \delta \gamma}{(1+\epsilon)} \, ,
  \qquad
  q = \frac{(1+\epsilon)^2}{\delta \gamma^2} \, .
  \]
  Let $\hat \vw$ be defined as in Lemma \ref{lmm:vec-compress-round}. Then, by Lemma \ref{lmm:vec-compress} we have that $\P_{\hat\vw}[\hat\vw \in \sA] \leq 1 - \delta$.
  We define $\sG = \{f_{\hat \vw}: \hat \vw \in \sA\}$. 
  Note that the mapping from $f_\vw$ to $f_{\hat \vw}$ fails (i.e., $\hat \vw \notin \sA$)
  with probability at most $\delta$, thus  corollary \ref{cor:bound-main} yields
  \[
  L_0^\epsilon(f_{\hat \vw}) \leq 
  \widehat L_\gamma^\epsilon(f_\vw) + \ord \parents{\sqrt{\frac{(1+\epsilon)^2 \log (n )\log \parents{\frac{4n \delta \gamma}{(1+\epsilon)}} }{\delta \gamma^2 m}}} + \delta 
  = \widehat L_\gamma^\epsilon(f_\vw) + \tilde \ord \parents{\sqrt{\frac{(1+\epsilon)^2}{\delta \gamma^2 m}}} + \delta 
  \]
  with high probability. 
  Then, we choose $\delta = ((1+\epsilon)^2/\gamma^2 m)^{1/3}$ which leads to 
  \[
  L_0^\epsilon(f_{\hat \vw}) \leq 
  \widehat L_\gamma^\epsilon(f_\vw) + \tilde \ord \parents{\parents{\frac{(1+\epsilon)^2}{\gamma^2 m}}^{1/3}} \, 
  \]
  with high probability.
\end{proof}

\begin{lemma}\label{lmm:sparsify-round-linear}
  Given an effectively $\bar s$-sparse vector $\vw \in \ball_{1,1}^n$, let us define $\vw' \in \ball_{1,1}^n$ as the $s$-sparse vector whose non-zero entries are the $s$-largest absolute entries of  $\vw$. In addition, the vector $\hat \vw$ is obtained by rounding each entry of $\vw'$ to the nearest multiple of $\gamma/s(1+\epsilon)$.
    % \item $\hat \vw$ is obtained by rounding each entry of $\tilde \vw$ to the nearest multiple of $\frac{\gamma}{2s(1+ \epsilon)}$. 
  % \end{itemize}
  If we choose $s = \bar s (1+\epsilon)/2\gamma$ then 
  \[
  \forall \vx \in \ball_{\infty,1}^n, y \in \sY : \quad 
  \abs{\margin_\epsilon(f_\vw; \vx, y) - \margin_\epsilon(f_{\hat \vw}; \vx, y) } \leq \gamma\, .
  \]
\end{lemma}
\begin{proof} (of Lemma \ref{lmm:sparsify-round-linear}) 
Let us first bound how much does sparsifying $\vw$ affects inner products, that is 
  \[
    \abs{\inner{\vw}{\vx} - \inner{\vw'}{\vx}} \leq \norm{\vw - \vw'}_1 \norm{\vx}_\infty \leq \norm{\vw - \vw'}_1 \, .
  \]
  This distorts the adversarial margin as follows:
  \begin{align*}
    \abs{\margin_\epsilon(f_\vw; \vx, y) - \margin_\epsilon(f_{\vw'}; \vx, y) } 
    & \leq \abs{\inner{\vw}{\vx} - \inner{\vw'}{\vx}} + \epsilon \abs{\norm{\vw}_1 - \norm{\vw'}_1} 
    \\
    & \leq \norm{\vw - \vw'}_1 + \epsilon \norm{\vw - \vw'}_1 
    & \text{(triangle inequality)}\\
    &= (1 + \epsilon) \norm{\vw - \vw'}_1 \\
    &\leq (1 + \epsilon) \frac{1}{4s}\norm{\vw}_{1/2} &\text{(Lemma \ref{lmm:sparse-approx-error})}\\
    &\leq (1 + \epsilon) \frac{\bar s}{4s}\norm{\vw}_1 &\text{(Definition of effective sparsity)}\\
    &= \gamma/2 \,. &\text{(Choice of $s$)}
  \end{align*}
  Similarly,
  \begin{align*}
    \abs{\margin_\epsilon(f_{\vw'}; \vx, y) - \margin_\epsilon(f_{\hat \vw}; \vx, y) } 
    &\leq (1 + \epsilon) \norm{\vw' - \hat \vw}_1 \\
    &\leq (1 + \epsilon) s \frac 12 \parents{\frac{\gamma}{s (1+\epsilon)} } \\
    &= \gamma/2 \,. 
  \end{align*}
  Putting all together we get
  \begin{align*}
    \abs{\margin_\epsilon(f_\vw; \vx, y) - \margin_\epsilon(f_{\hat \vw}; \vx, y) }
    &\leq \abs{\margin_\epsilon(f_\vw; \vx, y) - \margin_\epsilon(f_{\vw'}; \vx, y) } + \abs{\margin_\epsilon(f_{\vw'}; \vx, y) - \margin_\epsilon(f_{\hat \vw}; \vx, y) } \\
    &\leq  \gamma/2 + \gamma/2\, = \gamma.
  \end{align*}
\end{proof}
\begin{proof} (of Theorem \ref{thm:linear-sparse-bound})
  Let $\sA$ be the set of vectors with at most $\bar s (1+\epsilon)/2\gamma$ non-zero entries, where each entry is a multiple of $\gamma/s(1+\epsilon)$ between $-1$ and $1$.
  Then $\card{\sA} = r^q$ with
  \[
  r = \frac{2}{\gamma / s(1+\epsilon)}
  = \frac{2}{\gamma^2 / 2\bar s(1+\epsilon)^2} 
  = \frac{4 \bar s (1+\epsilon)^2}{\gamma^2}\, ,
  \qquad
  q = s = \bar s (1+\epsilon)/2\gamma \, .
  \]
  Let $\sG = \{f_{\hat \vw}: \hat \vw \text{ is defined as in Lemma \ref{lmm:sparsify-round-linear} with $\vw \in \ball_{1,1}^n$}\}$. Then, by Lemma \ref{lmm:sparsify-round-linear} we know that $f_\vw$ is ($\gamma, \epsilon, \sS$)-compressible via $\sG$, thus Theorem \ref{thm:bound-main} yields 
  \[
  L_0^\epsilon(f_{\hat \vw}) \leq 
  \widehat L_\gamma^\epsilon(f_\vw) + \ord \parents{\sqrt{\frac{2 \bar s(1+\epsilon) \log \parents{\frac{4 \bar s (1+\epsilon)^2}{\gamma^2}} }{\gamma m}}}  
  = \widehat L_\gamma^\epsilon(f_\vw) + \tilde \ord \parents{\sqrt{\frac{(1+\epsilon) \bar s}{\gamma m}}}
  \]
  with high probability. 
\end{proof}
\begin{proof} (of Lemma \ref{lmm:lipschitz})
Since $\phi$ is $1$-Lipschitz we have that for any vector $\vw$ of the same size as $\veta$ it holds
  \[
  \abs{\phi(\inner{\vw}{\vx}) - \phi(\inner{\vw}{\vx + \veta})} 
  \leq \abs{\inner{\vw}{\veta}}
  \leq \norm{\vw}_1 \norm{\veta}_\infty
  \leq \epsilon\norm{\vw}_1 \, .
  \]
  This proves the first inequality of the lemma.
  Similarly, for any $\vw$ and $\hat \vw$ it follows 
  \[
  \abs{\phi(\inner{\vw}{\vx}) - \phi(\inner{\hat \vw}{\vx})} 
  \leq \abs{\inner{\vw - \hat \vw}{\vx}}
  \leq \norm{\vw - \hat \vw}_1 \norm{\vx}_\infty
  \leq B \norm{\vw - \hat \vw}_1 \, ,
  \]
  thus implying the second inequality.
\end{proof}
\begin{proof}( of Lemma \ref{lmm:nn-layer-capacity})
Since $\mW$ is effectively joint sparse we can bound $\norm{\mW - \bar \mW}_{1, \infty}$ as follows
  \begin{align*}
    \norm{\mW - \bar \mW}_{1, \infty} 
    &\leq \frac{1}{s_2}\norm{\mW}_{1,1} &\text{(Lemma \ref{lmm:sparse-approx-error})}\\
    &\leq \frac{\bar s_2}{s_2}\norm{\mW}_{1,\infty} \, . &\text{(Definition of effective joint sparsity)}
  \end{align*}
  Similarly, since the remaining non-zero columns $\bar \mW$ are effectively sparse we get
  \begin{align*}
    \norm{\bar \mW - \tilde \mW}_{1, \infty} 
    &= \inf_{\mX : \norm{\mX}_{0,\infty}=s_1} \norm{\bar \mW - \mX}_{1,\infty} \\
    &\leq \frac{1}{4s_1} \norm{\bar \mW}_{1/2, \infty} &\text{(Lemma \ref{lmm:sparse-approx-error})}\\
    &\leq \frac{\bar s_1}{4s_1} \norm{\bar \mW}_{1, \infty} \, .
    &\text{(Definition of effective sparsity)}
  \end{align*}
  By the definition of $\hat \mW$ we have that $\norm{\tilde \mW - \hat \mW}_{1, \infty} \leq \gamma/3$. Combining all these statements, the choice of $s_1$ and $s_2$ (see Algorithm \ref{alg:compress}) yields
  \begin{align*}
    \norm{\mW - \hat \mW}_{1,\infty} 
    &\leq 
    \norm{\mW - \bar \mW}_{1, \infty} + 
    \norm{\bar \mW - \tilde \mW}_{1, \infty} + 
    \norm{\tilde \mW - \hat \mW}_{1, \infty}  \\
    &\leq \frac{\bar s_1}{4s_1} \norm{\mW}_{1,\infty} + \frac{\bar s_2}{s_2} \norm{\mW}_{1,\infty} + \frac \gamma 3 \\
    &\leq \frac \gamma 3 + \frac \gamma 3 + \frac \gamma 3 = \gamma\, .
  \end{align*}
  It remains to bound the covering number of $\sW$ with the mixed $(1,\infty)$-norm, denoted by $\sN(\sW, \norm{\cdot}_{1,\infty}, \gamma/3)$.
  By definition, the set $\sW$ is composed of all matrices $\tilde \mW$ with at most $s_2$ non-zero columns, where each column has at most $s_1$ non-zero entries and $\ell_1$-norm not greater than one.
  Since any $\tilde \mW \in \sW$ has at most $s_2$ non-zero columns we get
  \begin{align*}
    \sN(\sW, \norm{\cdot}_{1,\infty}, \gamma/3) &\leq
    \binom{n_2}{s_2} \sN(\gamma/3, \ball_{1,1}^{n_1} \cap \ball_{0,s_1}^{n_1}, \norm{\cdot}_1)^{s_2} \\
    &\leq \binom{n_2}{s_2} \brackets{\binom{n_1}{s_1} \sN(\gamma/3, \ball_{1,1}^{s_1}, \norm{\cdot}_1)}^{s_2} \\
    &\leq \parents{\frac{e n_2}{s_2}}^{s_2} \brackets{\parents{\frac{e n_1}{s_1}}^{s_1} \sN(\gamma/3, \ball_{1,1}^{s_1}, \norm{\cdot}_1)}^{s_2} \\
    &\leq \parents{\frac{e n_2}{s_2}}^{s_2} \parents{\frac{e n_1 }{s_1}}^{s_1 s_2}  \parents{1 + \frac 6 \gamma}^{s_1 s_2} \, .
    % &\leq \parents{\frac{4 \gamma e n_1 }{3\bar s_1 \norm{\mW}_{1,\infty} } \parents{1 + \frac 6 \gamma}}^{\frac{3\norm{\mW}_{1,\infty} \bar s_1 }{4\gamma} } 
    % \parents{\frac{\gamma e n_2}{3 \bar s_2 \norm{\mW}_{1,\infty} }}^{\frac{3\norm{\mW}_{1,\infty} \bar s_2 }{\gamma} } 
  \end{align*}
  This leads to 
  \begin{align*}
    \sN(\sW, \norm{\cdot}_{1,\infty}, \gamma/3) 
    &\leq \tilde \ord \parents{ s_1 s_2 } 
    = \tilde \ord \parents{\norm{\mW}_{1,\infty}^2 \bar s_1 \bar s_2 / \gamma^2 } \, .
  \end{align*}
  choosing $\sC$ to be the covering set of $\sW$ completes the proof.
\end{proof}
\begin{proof} (of Theorem \ref{thm:nn-bound})
  Let us assume that $\phi$ is the ReLU-activation. 
  Then, due to its positive homogeneity property, we re-balance the network by setting $\norm{\mW^i}_{1,\infty} =1$ for all $i=1,\dots,d$ without altering the classification function. 
  For any given adversarial noise $\veta_1$ with $\ell_\infty$ norm bounded by $\epsilon$, let us re-define $\vx^i$ as in \eqref{eq:nn-def} but with $\vx^0 = \vx + \veta_1$. 
  Similarly, for another adversarial noise $\veta_2$ with $\ell_\infty$-norm bounded by $\epsilon$ and compressed matrices $\hat \mW^i$, let us define the error vector of the $i$-th layer $\veta^i$ in a recursive fashion, that is  
  $\veta^i \deq \phi(\mW^{i^\T} \vx^{i-1}) - \phi(\hat\mW^{i^\T} (\vx^{i-1} + \veta^{i-1}))$ for $i=1,\dots, d$ with $\veta^0 \deq \veta_2 - \veta_1$. 
  Note that $\norm{\veta^0}_\infty \leq 2 \epsilon$.
  With this definition of $\vx^i$, since 
  \[
  \norm{\phi(\mW^{i^\T} \vx^{i-1})}_\infty 
  \leq \norm{\mW^{i^\T} \vx^{i-1}}_\infty 
  \leq \norm{\mW^{i^\T}}_\infty \norm{\vx^{i-1}}_\infty
  = \norm{\mW^{i}}_{1,\infty} \norm{\vx^{i-1}}_\infty 
  \]
  we have that $\norm{\vx^i}_\infty \leq \norm{\vx^0}_\infty \prod_{j=1}^i \norm{\mW^j}_{1,\infty} \leq 1+ \epsilon$. 

  Our first goal is to bound $\norm{\veta^i}_\infty$ for $i=1,2,\dots,d$, which we do by induction. 
  For any $i>0$, let us assume that $\norm{\veta^{i-1}} \leq \epsilon^{i-1}$ where $\epsilon^{i-1}$ is some positive value. 
  Given some $\epsilon^i > \epsilon^{i-1}$, we compress $\mW^i$ as $\hat \mW^i = \mrm{MatrixCompress}((\epsilon^i - \epsilon^{i-1})/(1+ \epsilon + \epsilon^{i-1}), \mW^i)$. Then, using Lemma \ref{lmm:lipschitz} we get
  \begin{align*}
    \norm{\veta^{i}}_\infty 
    &= \norm{\phi(\mW^{i^\T} \vx^{i-1}) - \phi(\hat\mW^{i^\T} (\vx^{i-1} + \veta^{i-1}))}_\infty \\
    &= \quad \norm{\phi(\mW^{i^\T} \vx^{i-1}) - \phi(\mW^{i^\T} (\vx^{i-1} + \veta^{i-1}))}_\infty 
    \\ &\quad +    
    \norm{\phi(\mW^{i^\T} (\vx^{i-1} + \veta^{i-1})) - \phi(\hat\mW^{i^\T} (\vx^{i-1} + \veta^{i-1}))}_\infty\\   
    &\leq \norm{\mW^i}_{1,\infty} \norm{\veta^{i-1}}_\infty + \norm{\mW^i - \hat \mW^i}_{1,\infty} \norm{\vx^{i-1} + \veta^{i-1}}_\infty &\text{(Lemma \ref{lmm:lipschitz})}\\
    &\leq \epsilon^{i-1} + \norm{\mW^i - \hat \mW^i}_{1,\infty} (1 + \epsilon + \epsilon^{i-1}) \\
    &\leq \epsilon^i \, . &\text{(Definition of $\hat \mW^i$)}
  \end{align*}
  Given $y$ and $f_\mW$, let us define $\tilde f_\mW(\vx) \deq [f_\mW(\vx)]_{j\neq y}$. By setting $\epsilon^0 \deq 2 \epsilon$ and $\epsilon^d \deq \gamma/2$ we get
  \begin{align*}
    &\abs{\margin_\epsilon(f_\mW; \vx, y) - \margin_\epsilon(f_{\hat \mW}; \vx, y) }
    \\
    &= \abs{
      [f_\mW(\vx +\veta_1)]_y - \max_{j\neq y}[f_\mW(\vx +\veta_1)]_j - 
      [f_{\hat\mW}(\vx +\veta_2)]_y + \max_{j\neq y}[f_{\hat\mW}(\vx +\veta_2)]_j
      }
    \\
    &= \abs{
      [f_\mW(\vx +\veta_1)]_y - \norm{\tilde f_\mW(\vx +\veta_1)}_\infty - 
      [f_{\hat\mW}(\vx +\veta_2)]_y + \norm{\tilde f_{\hat\mW}(\vx +\veta_2)}_\infty
      } 
    \\
    &\leq \abs{[f_\mW(\vx +\veta_1)]_y - [f_{\hat\mW}(\vx +\veta_2)]_y }
    + \abs{\norm{\tilde f_\mW(\vx +\veta_1)}_\infty - 
      \norm{\tilde f_{\hat\mW}(\vx +\veta_2)}_\infty}
    \\
    &\leq \norm{f_\mW (\vx +\veta_1) - f_{\hat \mW} (\vx +\veta_2)}_{\infty} + \norm{\tilde f_\mW (\vx +\veta_1) - \tilde f_{\hat \mW} (\vx +\veta_2)}_{\infty} 
    \\
    &\leq 2\norm{f_\mW (\vx +\veta_1) - f_{\hat \mW} (\vx +\veta_2)}_{\infty}
    \\
    &= 2\norm{\veta^{d}}_{\infty} 
    \leq \gamma
    \, .
  \end{align*}
  We are free to choose $\epsilon^1, \dots, \epsilon^{d-1}$ without loosing this bound on $\abs{\margin_\epsilon(f_\mW; \vx, y) - \margin_\epsilon(f_{\hat \mW}; \vx, y) }$, as long as $\epsilon^i > \epsilon^{i-1}$. 
  However, the choice of these values will determine the sample complexity of the compressed function class. 
  We choose these parameters as follows
  \[
  \epsilon^0 \deq 2 \epsilon \, , \qquad 
  \epsilon^i \deq \epsilon^{i-1} + \frac{ \sqrt{\bar s_1^i \bar s_2^i} }{ \sum_{j=1}^d \sqrt{\bar s_1^j \bar s_2^j} } (\gamma/2 - 2 \epsilon) \, .
  \]
  This rule allocates more error to the layers with more effective parameters 
  \footnote{A naive way of choosing $\epsilon^i$ like $\epsilon^i \deq i(\gamma/2 - 2 \epsilon)/d + 2 \epsilon$ will lead sample complexity of $\ord(d^2)$ instead of $\ord(d)$.}. 
  Note that this selection implies $\epsilon^d  = \gamma/2$ and $\epsilon^i > \epsilon^{i-1}$, so $f_\mW$ is $(\gamma, \epsilon, \sS)$-compressible via $\sG = \{f_{\hat \mW}: \hat \mW = \mrm{MatrixCompress}((\epsilon^i - \epsilon^{i-1})/(1+ \epsilon + \epsilon^{i-1}), \mW)\}$.
  In the same manner as in Lemma \ref{lmm:lipschitz}, for all $i=1,\dots, d$ let us define $\sC^i$ to be the set of all possible $\hat \mW^i$. 
  With this choice the logarithm of the logarithm of the cardinality of the compressed function class is 
  \begin{align*}
    \log \card{\sG} &= \log \prod_{i=1}^d \card{\sC^i}
    = \sum_{i=1}^d \log \card{\sC^i} \\
    &\leq \tilde\ord\parents{\sum_{i=1}^d \bar s_1^i \bar s_2^i (1+ \epsilon + \epsilon^{i-1})^2 / (\epsilon^i - \epsilon^{i-1})^2 } \\
    &\leq \tilde\ord\parents{\sum_{i=1}^d \frac{ \bar s_1^i \bar s_2^i (1+\epsilon + \gamma/2 - 2 \epsilon)^2 \parents{ \sum_{j=1}^d \sqrt{\bar s_1^j \bar s_2^j} }^2 }{ \parents{(\gamma/2 - 2 \epsilon )\sqrt{\bar s_1^i \bar s_2^i}}^2 } }\\
    &= \tilde\ord\parents{\sum_{i=1}^d \frac{ (1 + \gamma/2 -  \epsilon)^2 \parents{ \sum_{j=1}^d \sqrt{\bar s_1^j \bar s_2^j} }^2 }{ (\gamma/2 - 2 \epsilon )^2 } }
    \\
    &= \tilde\ord\parents{d \parents{\frac{ 1 + \gamma/2 -  \epsilon }{ \gamma/2 - 2 \epsilon }}^2 \parents{ \sum_{j=1}^d \sqrt{\bar s_1^j \bar s_2^j} }^2 }\,.
  \end{align*}
  % 
  % 
  % Then $\log \card{\hat \sH}  = \sum_{j=1}^d \log \card{\sW^j} \leq \tilde \ord \parents{\sum_{j=1}^d \norm{\mW^j}_{1,\infty}^2 \bar s_1^j \bar s_2^j / \beta^2  }$
  Finally, we apply Theorem \ref{thm:bound-main}, yielding
  \[
    L_0^\epsilon(f_{\hat \mW}) \leq \widehat L_\gamma^\epsilon(f_\mW) + \tilde \ord\left(\sqrt{
    \frac{d}{m} \parents{\frac{ 1 + \gamma/2 -  \epsilon }{ \gamma/2 - 2 \epsilon }}^2 \parents{ \sum_{j=1}^d \sqrt{\bar s_1^j \bar s_2^j} }^2
    } \right) \, .
  \]
\end{proof}
%%%%%%%%%%%%%%%%%%%%%%%%%%%%%%%%%%%%%%%%%%%%%%%%%%%%%%%%%%%%%%%
\end{document}